\documentclass[letterpaper]{article}
\usepackage{proceed2e}
\usepackage[margin=1in]{geometry}
\pdfoutput=1
\usepackage{times}

\title{Causal Consistency of Structural Equation Models}

\author{ {\bf Paul K.~Rubenstein{$^{*12}$}, Sebastian Weichwald{$^{*13}$}, Stephan Bongers$^{4}$, Joris M.~Mooij$^{4}$} \\
{\bf  Dominik Janzing$^{1}$, Moritz Grosse-Wentrup$^{1}$, Bernhard Sch\"olkopf$^{1}$}\\
$^*$Equal contribution \\
$^{1}$Empirical Inference, MPI for Intelligent Systems,
$^{2}$Machine Learning Group, University of Cambridge,\\ $^{3}$Max Planck ETH Center for Learning Systems,
$^{4}$Informatics Institute, University of Amsterdam\\
}

\makeatletter
\let\oldsection=\section
\renewcommand{\section}{\@ifstar\@ssection\@section}
\newcommand{\@ssection}[1]{\oldsection*{\MakeUppercase{#1}}}
\newcommand{\@section}[1]{\oldsection{\MakeUppercase{#1}}}
\let\oldsubsection=\subsection
\renewcommand{\subsection}{\@ifstar\@ssubsection\@subsection}
\newcommand{\@ssubsection}[1]{\oldsubsection*{\MakeUppercase{#1}}}
\newcommand{\@subsection}[1]{\oldsubsection{\MakeUppercase{#1}}}
\makeatother

\usepackage{tikz}
\usetikzlibrary{arrows.meta}
\usetikzlibrary{calc}
\usetikzlibrary{shapes}

\tikzstyle{green}=[fill=green!50!black,draw=green!50!black,color=green!50!black]
\tikzstyle{red}=[fill=red!50!black,draw=red!50!black,color=red!50!black]

\definecolor{docolour}{gray}{.75}

\tikzstyle{keep}=[circle,draw,minimum size=2em]
\tikzstyle{drop1}=[circle,fill=green!50!white]
\tikzstyle{drop2a}=[circle,fill=red!50!white]
\tikzstyle{drop2b}=[circle,fill=blue!50!white]
\tikzstyle{micro}=[circle,draw,minimum size=1em,fill=docolour]
\newcommand\dropNode[1]{\tikz[baseline={([yshift=-.6ex]current bounding box.center)}]\node[#1] at(0,0) {};}
\usetikzlibrary{shadings}
\pgfdeclarelayer{background}
\pgfsetlayers{background,main}

\usepackage{ifthen}

\usepackage{subcaption}

\usepackage{enumitem}

\usepackage{amsmath}
\usepackage{amsfonts}
\usepackage{amssymb}
\usepackage{mathtools}
\DeclareMathOperator{\doop}{do}
\DeclareMathOperator{\pa}{pa}

\newcommand{\nulli}{\varnothing}
\newcommand{\indexset}{\mathbb{I}}
\usepackage{amsthm}
\newtheorem{definition}{Definition}
\newtheorem{example}[definition]{Example}
\newtheorem{lemma}[definition]{Lemma}
\newtheorem{theorem}[definition]{Theorem}

\mathchardef\ordinarycolon\mathcode`\:
\mathcode`\:=\string"8000
\begingroup \catcode`\:=\active
  \gdef:{\mathrel{\mathop\ordinarycolon}}
\endgroup

\usepackage{microtype}

\usepackage{stackengine}
\usepackage{stix}

\usepackage[square,numbers,sort&compress]{natbib}
\usepackage{multibib}
\newcites{appendix}{References}
\bibliographystyleappendix{abbrvnat}

\begin{document}

\maketitle

\begin{abstract}
Complex systems can be modelled at various levels of detail.
Ideally, causal models of the same system should be consistent with one another in the sense that they agree in their predictions of the effects of interventions.
We formalise this notion of consistency in the case of Structural Equation Models (SEMs) by introducing \emph{exact transformations} between SEMs.
This provides a general language to consider, for instance, the different levels of description in the following three scenarios:
(a)~models with large numbers of variables versus models in which the `irrelevant' or unobservable variables have been marginalised out;
(b)~micro-level models versus macro-level models in which the macro-variables are aggregate features of the micro-variables;
(c)~dynamical time series models versus models of their stationary behaviour.
Our analysis stresses the importance of well specified interventions in the causal modelling process and sheds light on the interpretation of cyclic SEMs.
\end{abstract}

\section{Introduction}

Physical systems or processes in the real world are complex and can be understood at various levels of detail.
For instance, a gas in a volume consists of a large number of molecules.
But instead of modelling the motions of each particle individually (micro-level), we may choose to consider macroscopic properties of their motions such as temperature and pressure.
Our decision to use such macroscopic properties is first necessitated by practical considerations.
Indeed, for all but extremely simple cases, making a measurement of all the individual molecules is practically impossible and our resources insufficient for modelling the ${\sim}10^{22}$ particles present per litre of ideal gas.
Furthermore, the decision for a macroscopic description level is also a pragmatic one: if we only wish to reason about temperature and pressure, a model of $10^{22}$ particles is ill-suited.

Statistical physics explains how higher-level concepts such as temperature and pressure arise as statistical properties of a system of a large number of particles, justifying the use of a macro-level model as a useful transformation of the micro-level model~\citep{Balian}.
However, in many cases aggregate or indirect measurements of a complex system form the basis of a macroscopic description of the system, with little theory to explain whether this is justified or how the micro- and macro-descriptions stand in relation to each other.

Due to deliberate modelling choice or the limited ability to observe a system, differing levels of model descriptions are ubiquitous and occur, amongst possibly others, in the following three settings:
\vspace{-.5em}
\begin{itemize}[noitemsep]
    \item[(a)] Models with large numbers of variables versus models in which the `irrelevant' or unobservable variables have been marginalised out \citep{bongers2016structural}; e.\,g.\ modelling blood cholesterol levels and risk of heart disease while ignoring other blood chemicals or external factors such as stress.

    \item[(b)] Micro-level models versus macro-level models in which the macro-variables are aggregate features of the micro-variables \citep{simon1961aggregation,iwasaki1994causality,hoel2013quantifying,chalupka2015visual,chalupka2016multi}; e.\,g.\ instead of modelling the brain as consisting of $100$ billion neurons it can be modelled as averaged neuronal activity in distinct functional brain regions.

    \item[(c)] Dynamical time series models versus models of their stationary behaviour \citep{fisher1970correspondence,iwasaki1994causality,dash2001caveats,lacerda2012discovering,mooij2013ode,mooij2013cyclic}; e.\,g.\ modelling only the final ratios of reactants and products of a time evolving chemical reaction.
\end{itemize}
\vspace{-.5em}
In the context of causal modelling, such differing model levels  should be consistent with one another in the sense that they agree in their predictions of the effects of interventions. The particular causal models we focus on in this paper are Structural Equation Models (SEMs, Section~\ref{sec:SEMs}, Section~\ref{sec:sem-for-causal-modelling}) \citep{spirtes2000causation,pearl2009causality}.

In Section~\ref{sec:sem-transformation}, we introduce the notion of an exact transformation between two SEMs, providing us with a general framework to evaluate when two models can be thought of as causal descriptions of the same system.
An important novel idea of this paper is to explicitly make use of a natural ordering on the set of interventions.
On a high level, if an SEM can be viewed as an exact transformation of another SEM, we are provided with an explicit correspondence between the two models in such a way that causal reasoning on both levels is consistent.
We discuss this notion of consistency in detail in Sections~\ref{subsec:causal-interpretation-transformation} and~\ref{sec:wrong}.

In Section~\ref{sec:example-transformations} we apply this mathematical framework and prove the exactness of transformations belonging to each of the three categories listed above, with practical implications for the following questions in causal modelling:
When can we model only a subsystem of a more complex system?
When does a micro-level system admit a causal description in terms of macro-level features?
How do cyclic SEMs arise?
The fact that these distinct problems can all be considered using the language of transformations between SEMs demonstrates the generality of our approach.
We close in Section~\ref{sec:questions} with a discussion.

\subsection{A historical motivation: Cholesterol and Heart Disease}\label{sec:cholesterol}

\begin{figure}
\begin{subfigure}{.45\linewidth}
\center\
\begin{tikzpicture}
\node (d1) at(0,0.0) {diet};
\node (LDL) at(1.5,0.3) {LDL};
\node (HDL) at(1.5,-0.3) {HDL};
\node (HD) at(3,0) {HD};

\draw[->,thick] (d1) -- (HDL);
\draw[->,thick] (d1) -- (LDL);
\draw[->,red,thick] (LDL) -- node[above,yshift=-1] {$-$} (HD);
\draw[->,green,thick] (HDL) -- node[below] {$+$} (HD);
\end{tikzpicture}
\caption{}\label{fig:cholesterol:b}
\end{subfigure}
\hfill
\begin{subfigure}{.45\linewidth}
\center\
\begin{tikzpicture}
\node (d1) at(0,0.) {diet};
\node (LHDL) at(1.5,0) {TC};
\node (HD) at(3,0) {HD};

\draw[->,thick] (d1) -- (LHDL);
\draw[->,red,dash pattern= on 3pt off 3pt,ultra thick] (LHDL) -- node[above,yshift=-1] {$-$} (HD);
\draw[->,green,dash pattern= on 3pt off 3pt,dash phase=3pt,ultra thick] (LHDL) -- node[below] {$+$} (HD);
\end{tikzpicture}
\caption{}\label{fig:cholesterol:a}
\end{subfigure}
\caption{As illustrated by~(a), the current consensus is that LDL (resp.\ HDL) has a negative (resp.\ positive) effect on heart disease (HD). Considering TC = LDL + HDL to be a causal variable as in~(b) leads to problems: two diets promoting raised LDL levels and raised HDL levels have the same effect on TC but opposite effects on heart disease. Hence different studies come to contradictory conclusions about the effect of TC on heart disease.}
\label{fig:cholesterol}
\end{figure}
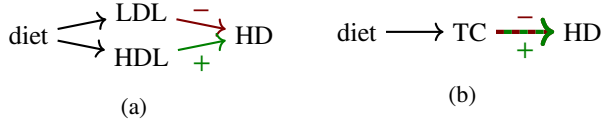

In the following we give an example of the problems that can arise when there exists no consistent correspondence between two causal models, i.\,e.\ neither model can be viewed as an exact transformation of the other. This example falls into category (b) of the differing model levels listed above and was used by~\cite{spirtes2004causal} to illustrate problems in the causal modelling process.

Historically, the level of total cholesterol in the blood (TC) was thought to be an important variable in determining risk of heart disease (HD).
To investigate this, different experiments were carried out in which patients were assigned to different diets in order to raise or lower TC\@.
Conflicting evidence was found by different experiments: some found that higher TC had the effect of lowering HD, while others found the opposite (cf.\ Figure~\ref{fig:cholesterol:a}) \citep{truswell2010cholesterol,steinberg2011cholesterol}.

From our point of view, this problem (seemingly conflicting studies) arose from trying to perform an `invalid' transformation of the `true' underlying model (cf.\ Figure~\ref{fig:cholesterol:b}).
According to the American Heart Association, the current scientific consensus is that the two types of blood cholesterol, low-density lipoprotein (LDL) and high-density lipoprotein (HDL), have a negative and positive effect on HD respectively.
Assigning diets that raise LDL or HDL both raise TC but have different effects on HD\@.
It is therefore not possible to transform the model in Figure~\ref{fig:cholesterol:b} into the model in Figure~\ref{fig:cholesterol:a} without leading to conflict:
in order to reason about the causes of HD we need to consider the variables LDL and HDL separately.

\section{Structural Equation Models}\label{sec:SEMs}

SEMs are a widely used framework in causal modelling, with applications in neuroscience, economics and the social sciences \citep{pearl2009causality,bollen2014structural}. In this section we introduce them as an abstract mathematical object; in Section~\ref{sec:sem-for-causal-modelling} we describe their use as a causal modelling tool.
Readers already familiar with SEMs should note that our definition is more general and deviates from the standard definition of SEMs in the following ways: we do not require that all possible perfect interventions be modelled; we do not assume independence of exogenous variables;\footnote{Exogenous variables are also referred to as \emph{noise variables} in the literature. Our relaxation of the assumption of independent exogenous variables means our models may be considered a type of semi-Markovian causal model.} and we do not require acyclicity.

\begin{definition}[Structural Equation Model (SEM)]
    Let $\indexset_X$ be an index set.
    An SEM $\mathcal{M}_X$ over  variables ${X = (X_i : i\in\indexset_X )}$ taking value in $\mathcal{X}$ is a triple $\left(\mathcal{S}_X, \mathcal{I}_X, \mathbb{P}_{E} \right)$ where
    \vspace{-.5em}
    \begin{itemize}[noitemsep]
        \item $\mathcal{S}_X$ is a set of structural equations, i.\,e.\ it is a set of equations $X_i = f_i\left( X , E_i \right)\ $ for $i \in \indexset_X$;
        \item ($\mathcal{I}_X, \leq_X)$ is a subset of all perfect interventions equipped with a natural partial ordering (see below), i.\,e.\ it is an index set where each index corresponds to a particular perfect intervention on some of the $X$ variables;
        \item $\mathbb{P}_{E}$ is a distribution over the exogenous variables $E = ( E_i : i\in\indexset_X)$;
        \item with $\mathbb{P}_E$-probability one, under any intervention ${i \in \mathcal{I}_X}$ there is a unique solution $x\in\mathcal{X}$ to the intervened structural equations. This ensures that for any intervention ${i \in \mathcal{I}_X}$, $\mathcal{M}_X$ induces a well-defined distribution over $\mathcal{X}$.%
\footnote{That is, with probability one over the exogenous variables $E$, for each draw $E=e$ there exists a unique value $x\in \mathcal{X}$ such that $e$ and $x$ satisfy the intervened structural equations. The distribution of $E$ in conjunction with $\mathcal{S}_X$ then implies a distribution over $\mathcal{X}$ for each intervention $i \in \mathcal{I}_X$ via these unique solutions.
If the SEM is acyclic, this is always satisfied; we impose this condition because we also consider \emph{cyclic} SEMs \cite{bongers2016structural}.}
    \end{itemize}
    \vspace{-.5em}
\end{definition}

In an SEM, each $X_i$ is a function of the $X$-variables and the exogenous variable $E_i$.
In this mathematical model, a perfect intervention on a single variable $\doop(X_i=x_i)$ is realised by replacing the structural equation for variable $X_i$ in $\mathcal{S}_X$ with $X_i = x_i$.
Perfect interventions on multiple variables, e.g. $\doop(X_i=x_i, X_j=x_j)$, are similarly realised by replacing the structural equations for each variable individually.
Elements of $\mathcal{I}_X$ correspond to perfectly intervening on a subset of the $X$ variables, setting them to some particular combination of values.

$\mathcal{I}_X$ has a natural partial ordering in which, for interventions ${i, j \in \mathcal{I}_X}$, ${i\leq_X j}$ if and only if $i$ intervenes on a subset of the variables that $j$ intervenes on and sets them equal to the same values as $j$.
For example, ${\doop(X_i=x_i) \leq_X \doop(X_i=x_i, X_j=x_j)}$.\footnote{Informally, this means that $j$ can be performed after $i$ without having to change or undo any of the changes to the structural equations made by $i$.
Not all pairs of elements must be comparable: for instance, if $i= \doop(X_1=x_1)$ and $j = \doop(X_2=x_2)$, then neither $i\leq_X j$ nor $j \leq_X i$.} The observation that this structure is important is a contribution of this paper. We make crucial use of it in the next section.

The purpose of the following example is to illustrate how SEMs are written in our notation and to provide and example of a restricted set of interventions $\mathcal{I}_X$.

\begin{example}\label{example1}
Consider the following SEM defined over the variables $\{ B_1,B_2,L \}$
\begin{align*}
\mathcal{S}_X = \big\{ & B_1 = E_1,\ B_2 = E_2,\ L = \operatorname{OR}(B_1,B_2,E_3) \big\} \\
\mathcal{I}_X = \big\{ & \nulli,\ \doop(B_1=0),\ \doop(B_2=0),\\
& \doop(B_1=0,B_2=0) \big\}, \\
\{ E_1,E_2,E_3 \} & \overset{\text{iid}}{\sim} \mathrm{Bernoulli}(0.5)
\end{align*}
where by the element $\nulli \in \mathcal{I}$ we denote the null-intervention corresponding to the unintervened SEM\@.
\end{example}

\section{SEMs for Causal Modelling}\label{sec:sem-for-causal-modelling}

In addition to being abstract mathematical objects, SEMs are used in causal modelling to describe distributions of variables and how they change under interventions \citep{pearl2009causality}.
The $\doop$-interventions as abstract manipulations of SEMs are understood as corresponding to actual (or potentially only hypothetical) physical implementations in the real world, i.\,e.\ the model is `rooted in reality'.
For instance, if a binary variable $B_1$ in an SEM reflects whether a light bulb is emitting light, then $\doop(B_1=0)$ could be achieved by flipping the light switch or by removing the light bulb.

The SEM in Example~\ref{example1} could be thought of as a simple causal model of two light bulbs $B_1$ and $B_2$ and the presence of light $L$ in a room with a window.
Suppose that we have no access to the light switch and there are no curtains in the room but that we can intervene by removing the light bulbs.
We can model this restricted set of interventions by $\mathcal{I}_X$, i.\,e.\ the $\doop$-intervention on the SEM side ${\doop(B_1=0)}$  corresponds to removing the light bulb $B_1$.

The partial ordering of $\mathcal{I}_X$ corresponds to the ability to compose physical implementations of interventions. The fact that we can first remove light bulb $B_1$ (${\doop(B_1=0)}$) and then afterwards remove light bulb $B_2$ (resulting in the combined intervention ${\doop(B_1=0, B_2=0)}$)  is reflected in the partial ordering via the relation ${\doop(B_1=0) \leq_X \doop(B_1=0,B_2=0)}$.

\section{Transformations between SEMs}\label{sec:sem-transformation}

We now work towards our definition of an exact transformation between SEMs\@.
Our core idea is to analyse the correspondence between different levels of modelling by considering one model to be a transformation of the other.
We discuss in Section~\ref{subsec:causal-interpretation-transformation} how causal reasoning in two SEMs relate when one SEM can be viewed as an exact transformation of the other and in Section~\ref{sec:wrong} we illustrate what can go wrong when this is not the case.

\subsection{Distributions implied by an SEM}

Usually, a statistical model implies a single joint distribution over all variables once its parameters are fixed.
SEMs are different in that, once the parameters are fixed, an SEM implies a family of joint distributions over the random variables, one for each intervention.
That is, for each intervention $i \in \mathcal{I}_X$, the SEM $\mathcal{M}_X$ defines a distribution over $\mathcal{X}$ which we denote by $\mathbb{P}_X^{\doop(i)}$.
Throughout, we will denote the null-intervention corresponding to the unintervened setting by $\nulli\in \mathcal{I}_X$ .
We can write the poset of all distributions implied by the SEM $\mathcal{M}_X$ as
\[\mathcal{P}_X := \left( \left\{ \mathbb{P}_X^{\doop(i)} \enspace : \enspace i \in \mathcal{I}_X \right\}, \leq_X \right) \]
where $\leq_X$ is the partial ordering inherited from $\mathcal{I}_X$, i.\,e.\ ${\mathbb{P}_X^{\doop(i)} \leq_X \mathbb{P}_X^{\doop(j)} \iff i \leq_X j}$.%
\footnote{More formally, one would need to define $\mathcal{P}_X$ to be the poset of \emph{tuples} $\left(i, \mathbb{P}_X^{\doop(i)}\right)$ to avoid problems in the case that $\mathbb{P}_X^{\doop(i)} = \mathbb{P}_X^{\doop(j)}$ for some $i\not=_X j$. Doing so would not require a change to Definition~3 or affect the further results of this paper. To avoid notational burden in our exposition, we omit this treatment.}

Note that $\mathcal{P}_X$ contains all of the information in $\mathcal{M}_X$ about the different distributions implied by the SEM and, importantly, how they are related via the interventions.%
\footnote{For example, the distribution over the variables $X$ in the observational setting, $\mathbb{P}_X^\nulli$, changes to $\mathbb{P}_X^{\doop(i)}$ if we implement the intervention ${\doop(i)}$, and the partial ordering contains all information about which interventions can be composed.}

\subsection{Transformations of random variables}

Suppose we have a function ${\tau: \mathcal{X} \to \mathcal{Y}}$ which maps the variables of the SEM $\mathcal{M}_X$ to another space $\mathcal{Y}$.
Observe that since $X$ is a random variable, $\tau(X)$ is also a random variable.
For any distribution $\mathbb{P}_X$ on $\mathcal{X}$ we thus obtain the distribution of the variable $\tau(X)$ on $\mathcal{Y}$ as $\mathbb{P}_{\tau(X)} = \tau\left(\mathbb{P}_X\right)$ via the push-forward measure.

In particular, for each intervention $i \in \mathcal{I}_X$ we can define the induced distribution $\mathbb{P}_{\tau(X)}^{i} = \tau\left(\mathbb{P}_X^{\doop(i)}\right)$.
We can write the poset of distributions on $\mathcal{Y}$ that are induced by the original SEM $\mathcal{M}_X$ and the transformation $\tau$ as
\[\mathcal{P}_{\tau(X)} := \left( \left\{ \mathbb{P}_{\tau(X)}^{i} \enspace : \enspace i \in \mathcal{I}_X \right\}, \leq_X \right) \]
where $\leq_X$ is the partial ordering inherited from $\mathcal{P}_X$ (and in turn from $\mathcal{I}_X$).

$\mathcal{P}_{\tau(X)}$ is just a structured collection of distributions over $\mathcal{Y}$, indexed by interventions $\mathcal{I}_X$ on the $\mathcal{X}$-level; importantly, the indices are \emph{not} interventions on the $\mathcal{Y}$-level.

\subsection{Exact Transformations between SEMs}\label{sec:exact_transformation_sem}

Although $\mathcal{P}_{\tau(X)}$ is a poset of distributions over $\mathcal{Y}$, there does not necessarily exist an SEM $\mathcal{M}_Y$ over $\mathcal{Y}$ that implies it.
For instance, if there is some intervention ${i \in \mathcal{I}_X \setminus \{ \nulli \}}$ such that none of the variables $Y_i$ is constant under the distribution $\mathbb{P}_{\tau(X)}^{i}$, then $\mathbb{P}_{\tau(X)}^{i}$ could not possibly be expressed as arising from a $\doop$-intervention $j\in \mathcal{I}_Y \setminus \{\nulli\}$ in any SEM over~$\mathcal{Y}$.\footnote{This problem is elaborated upon in~\cite{eberhardt2016green}.}

The case in which there \emph{does} exist an SEM $\mathcal{M}_Y$ that implies $\mathcal{P}_{\tau(X)}$ is special, motivating our main definition.

\begin{definition}[Exact Transformations between SEMs]\label{def:exacttrafos}
Let $\mathcal{M}_X$ and $\mathcal{M}_Y$ be SEMs and $\tau: \mathcal{X} \to \mathcal{Y}$ be a function.
We say $\mathcal{M}_Y$ is an \emph{exact}  $\tau$-transformation of $\mathcal{M}_X$ if there exists a \emph{surjective order-preserving} map $\omega:\mathcal{I}_X\rightarrow \mathcal{I}_Y$ such that
\[ \mathbb{P}_{\tau(X)}^{i} = \mathbb{P}_Y^{\doop(\omega(i))} \quad \forall i \in \mathcal{I}_X \]
where $\mathbb{P}_{\tau(X)}^{i}$ is the distribution of the $\mathcal{Y}$-valued random variable $\tau(X)$ with $X \sim \mathbb{P}_X^{\doop(i)}$.
\end{definition}

Order-preserving means that ${i \leq_X j \implies \omega(i) \leq_Y \omega(j)}$.
It is important that the converse need not in general hold as this would imply that $\omega$ is injective,%
\footnote{Since ${\omega(i)=\omega(j) \iff \left(\omega(i) \leq_Y \omega(j)\right) \land \left(\omega(j) \leq_Y \omega(i)\right)}$, which, if the converse held, would imply that $\left(i \leq_X j\right) \land \left(j \leq_X i\right)$, which is equivalent to $i=j$.}
and hence also bijective.
This would constrain the ways in which $\mathcal{M}_Y$ can be `simpler' than $\mathcal{M}_X$.\footnote{For instance, if it were necessary that $\omega$ be bijective, Theorems~\ref{theorem:childless} and \ref{theorem:micro-macro} would not hold.}
That $\omega$ is surjective ensures that for any $\doop$-intervention $j \in \mathcal{I}_Y$ on $\mathcal{M}_Y$ there is at least one corresponding intervention on the $\mathcal{M}_X$ level, namely an element of $\omega^{-1}(\{j\}) \subseteq \mathcal{I}_X$.
The following two results follow immediately from the definition (cf.\ proofs in Appendix~\ref{first_properties:appendix}).

\begin{lemma}\label{lemma:elementary}
The identity mapping and permuting the labels of variables are both exact transformations.
\end{lemma}

This is a good sanity check; it would be problematic if this were not the case and the labelling of our variables mattered. Similarly, compositions of exact transformations are also exact.

\begin{lemma}[Transitivity of exact transformations]\label{theorem:transitivity}
    If $\mathcal{M}_Z$ is an exact $\tau_{ZY}$-transformation of $\mathcal{M}_Y$ and $\mathcal{M}_Y$ is an exact $\tau_{YX}$-transformation of $\mathcal{M}_X$, then $\mathcal{M}_Z$ is an exact $(\tau_{ZY}\circ\tau_{YX})$-transformation of $\mathcal{M}_X$.
\end{lemma}

The following theorem is a consequence of the fact that $\omega$ is order-preserving. This is a mathematical formalisation of the sense in which an exact transformation preserves causal reasoning, which will be elaborated upon in the next subsection.

\begin{theorem}[Causal consistency under exact transformations]\label{lemma:commuting}
Suppose that  $\mathcal{M}_Y$ is an exact $\tau$-transformation of $\mathcal{M}_X$ and $\omega$ is a corresponding surjective order-preserving mapping between interventions. Let $i,j \in \mathcal{I}_X$ be interventions such that $i\leq_X j$.
Then the following diagram commutes:\\
\begin{tikzpicture}[thick, every node/.style = {circle, minimum size=.5cm}]

{
  \node[draw=none](Px) {$\mathbb{P}_X$};
  \node[draw=none, right of=Px, xshift=2cm](PxInt) {$\mathbb{P}_X^{\doop(i)}$};
  \node[draw=none, right of=PxInt, xshift=2cm](PxInt2) {$\mathbb{P}_X^{\doop(j)}$};

  \node[draw=none, below of=Px, yshift=-1.5cm](Py) {$\mathbb{P}_Y$};
  \node[draw=none, right of=Py, xshift=2cm](PyInt) {$\mathbb{P}_Y^{\doop(\omega(i))}$};
  \node[draw=none, right of=PyInt, xshift=2cm](PyInt2) {$\mathbb{P}_Y^{\doop(\omega(j))}$};

  \draw[-{Latex[length=2mm,width=2mm]}](Px)--(Py);
  \draw[-{Latex[length=2mm,width=2mm]}](Px)--(PxInt);
  \draw[-{Latex[length=2mm,width=2mm]}](Py)--(PyInt);
  \draw[-{Latex[length=2mm,width=2mm]}](PxInt)--(PyInt);
  \draw[-{Latex[length=2mm,width=2mm]}](PxInt)--(PxInt2);
  \draw[-{Latex[length=2mm,width=2mm]}](PyInt)--(PyInt2);
  \draw[-{Latex[length=2mm,width=2mm]}](PxInt2)--(PyInt2);

  \node[draw=none, yshift=0.5cm](xInt) at ($(Px)!0.5!(PxInt)$){$\doop(i)$};
  \node[draw=none, yshift=0.5cm](xInt2) at ($(PxInt)!0.5!(PxInt2)$){$\doop(j)$};
  \node[draw=none, yshift=0.5cm](yInt) at ($(Py)!0.5!(PyInt)$){$\doop(\omega(i))$};
  \node[draw=none, yshift=0.5cm](yInt2) at ($(PyInt)!0.5!(PyInt2)$){$\doop(\omega(j))$};
  \node[draw=none, xshift=-0.5cm](tau) at ($(Px)!0.5!(Py)$){$\tau$};
  \node[draw=none, xshift=0.5cm](tauInt) at ($(PxInt)!0.5!(PyInt)$){$\tau$};
  \node[draw=none, xshift=0.5cm](tauInt2) at ($(PxInt2)!0.5!(PyInt2)$){$\tau$};
}
\end{tikzpicture}
\end{theorem}

\begin{proof}
Let $i,j\in\mathcal{I}_X$ be interventions with $i\leq_X j$.
The commutativity of the left square of the diagram follows immediately from the definition of an exact transformation.
It remains to be shown that the right square of the diagram commutes.
By definition we have that $\tau\left(\mathbb{P}_X^{\doop(i)}\right) = \mathbb{P}_Y^{\doop(\omega(i))}$ and $\tau\left(\mathbb{P}_X^{\doop(j)}\right) = \mathbb{P}_Y^{\doop(\omega(j))}$.
Thus, we only have to show that ${\mathbb{P}_Y^{\doop(\omega(i))} \leq_Y \mathbb{P}_Y^{\doop(\omega(j))}}$ as elements of $\mathcal{P}_Y$, i.\,e.\ that the arrow ${\mathbb{P}_Y^{\doop(\omega(i))} \xrightarrow{\doop(\omega(j))} \mathbb{P}_Y^{\doop(\omega(j))}}$ exists.
This follows from the order-preservingness of $\omega$.
\end{proof}

\subsection{Causal Interpretation of Exact Transformations}\label{subsec:causal-interpretation-transformation}

The notion of an exact transformation between SEMs was motivated by the desire to analyse the correspondence between two causal models describing the same system at different levels of detail.
The purpose of this section is to show that if one SEM can be viewed as an exact transformation of the other, then both can sensibly be thought of as causal models of the same system. In the following, we assume that $\mathcal{M}_Y$ is an exact $\tau$-transformation of $\mathcal{M}_X$ with $\omega$ the corresponding map between interventions.

Surjectivity of $\omega$ ensures that any intervention in $\mathcal{I}_Y$ can be viewed as an $\mathcal{M}_Y$-level representative of some intervention on the $\mathcal{M}_X$-level. Consequently, if $\doop$-interventions on the $\mathcal{M}_X$-level are in correspondence with physical implementations, then surjectivity of $\omega$ ensures that $\doop$-interventions on the $\mathcal{M}_Y$-level have at least one corresponding physical implementation, i.\,e.\ if $\mathcal{M}_X$ is `rooted in reality', then so is $\mathcal{M}_Y$.

Commutativity of the left hand part of the diagram ensures that the effects of interventions are consistently modelled by $\mathcal{M}_X$ and $\mathcal{M}_Y$.
Suppose we want to reason about the effects on the $\mathcal{M}_Y$-level caused by the intervention $j \in \mathcal{I}_Y$.
For example, we may wish to reason about how the temperature and pressure of a volume of gaseous particles is affected by being heated.
We could perform this reasoning by considering any corresponding $\mathcal{M}_X$-level intervention $i \in \omega^{-1}(\{j\})$ and considering the distribution this implies over $\mathcal{Y}$ via $\tau$.
In our example, this would correspond to considering how heating the volume of gas could be modelled by changing the motions of all the gaseous particles and then computing the temperature and pressure of the volume of particles.
Commutativity of the left hand part of the diagram implies that $\mathcal{M}_X$ and $\mathcal{M}_Y$ are consistent in the sense that $\mathcal{M}_Y$ allows us to immediately reason about the effect of the intervention $j\in\mathcal{I}_Y$ while being equivalent to performing the steps above.
That is, we can reason directly about temperature and pressure when heating a volume of gas without having to perform the intermediate steps that involve the microscopic description of the system.

Commutativity of the right hand side of the diagram ensures that once an intervention that fixes a subset of the variables has been performed, we can still consistently reason about the effects of further interventions on the remaining variables in $\mathcal{M}_X$ and $\mathcal{M}_Y$.
Furthermore, it ensures that compositionality of $\doop$-interventions on the $\mathcal{M}_X$-level carries over to the $\mathcal{M}_Y$-level, i.\,e.\ if the intervention $j$ on the $\mathcal{M}_X$-level can be performed additionally to the intervention $i$ in $\mathcal{M}_X$---that is, $i\leq_X j$---, then the same is true of their representations in $\mathcal{M}_Y$.

If $\mathcal{M}_X$ and $\mathcal{M}_Y$ are models of the same system and it has been established that $\mathcal{M}_Y$ is an exact $\tau$-transformation of $\mathcal{M}_X$ for some mapping $\tau$, then the commutativity of the whole diagram in Theorem~\ref{lemma:commuting} ensures that they are causally consistent with one another in the sense described in the preceding paragraphs.
If we wish to reason about the effects of interventions on the $\mathcal{Y}$-variables then it suffices to use the model $\mathcal{M}_Y$, rather than the (possibly more complex) model $\mathcal{M}_X$.
In particular, this means that we can view the $\mathcal{Y}$-variables as causal entities, rather than only functions of underlying `truly' causal entities.
Only if this is the case, causal statements such as `raising temperature increases pressure' or `LDL causes heart disease' are meaningful.

\subsection{What can go wrong when a transformation is not exact?}\label{sec:wrong}

In the previous section we argued that our definition of exact transformations between SEMs is a sensible formalisation of causal consistency.
In this section we will try to give the reader an intuition for why weakening the conditions of our definition would be problematic.
In particular we focus on the requirement that $\omega$ be order-preserving, which we view as one of the core ideas of our paper.

The requirement that $\omega$ be surjective is, as discussed above, required so that all interventions on the $\mathcal{M}_Y$-level have a corresponding intervention on the $\mathcal{M}_X$-level.
If we were to only require that $\omega$ be surjective (but not order-preserving), the observational distribution of $\mathcal{M}_X$ may be mapped to an interventional distribution of $\mathcal{M}_Y$, as illustrated by the following example (cf.\ Figure~\ref{fig:wrong-example} for an illustration).

\begin{figure}
\begin{subfigure}{.45\linewidth}
\center\
\begin{tikzpicture}
\node (x1) at(0,0) {$X_1$};
\node (x2) at(2,0) {$X_2$};
\node (x3) at(1,-1) {$X_3$};

\draw[->] (x1) -- (x3);
\draw[->] (x2) -- (x3);
\draw[dashed] (x1) to [bend left] (x2);
\end{tikzpicture}
\caption{SEM $\mathcal{M}_X$}
\end{subfigure}
\hfill
\begin{subfigure}{.45\linewidth}
\center\
\begin{tikzpicture}
\node (y1) at(0,0) {$Y_1\ {\color{gray}= X_1+X_2}$};
\node (y2) at(-.415,-1) {$Y_2\ {\color{gray}= X_3}$};

\draw[->] (-.88,-.25) -- (-.88,-.75);
\end{tikzpicture}
\caption{SEM $\mathcal{M}_Y$}
\end{subfigure}
\caption{Graphical illustration of parent-child relationships for the examples in Section~\ref{sec:wrong}. The micro-level model $\mathcal{M}_X$ depicted in (a) is to be transformed into the macro-level model $\mathcal{M}_Y$ depicted in (b) which is a coarser descriptions as in it only considers the sum of $X_1$ and $X_2$. In Section~\ref{sec:wrong} we give examples of what can go wrong if the transformation is not exact.}
\label{fig:wrong-example}
\end{figure}
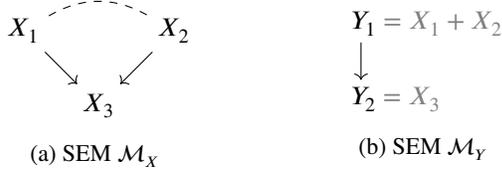

\begin{example}\label{example:wrong1}
Consider the SEM $\mathcal{M}_X=\{\mathcal{S}_X , \mathcal{I}_X, \mathbb{P}_E\}$ over $\mathcal{X}=\mathbb{R}^3$ where
\begin{align*}
\mathcal{S}_X = \big\{ & X_1 = E_1,\ X_2 = E_2,\ X_3 = X_1 + X_2 + E_3 \big\} \\
\mathcal{I}_X = \big\{ & \nulli,\ \doop(X_2=0),\ \doop(X_1=0,\, X_2=0)\big\}, \\
E_1 & \sim \mathbb{P}_{E_1},\ \,  E_2 = - E_1, \  \, E_3 \sim \mathbb{P}_{E_3}
\end{align*}
where $\mathbb{P}_{E_1}$ and $\mathbb{P}_{E_3}$ are arbitrary distributions.
Let ${\tau:\mathcal{X}\to\mathcal{Y}=\mathbb{R}^2}$ be the mapping such that
\begin{align*}
\tau\begin{pmatrix} x_1, x_2, x_3 \end{pmatrix}
= \begin{pmatrix} y_1, y_2\end{pmatrix}
= \begin{pmatrix} x_1 + x_2, x_3\end{pmatrix}
\end{align*}
Let $\mathcal{M}_Y =\{\mathcal{S}_Y , \mathcal{I}_Y, \mathbb{P}_F\}$ be an SEM over $\mathcal{Y}$ with
\begin{align*}
\mathcal{S}_Y = \big\{ & Y_1 = F_1,\ Y_2 = Y_1 + F_2 \big\} \\
\mathcal{I}_Y = \big\{ & \nulli,\ \doop(Y_1=0) \big\}, \\
F_1 \sim \mathbb{P}_{E_1},&  \  \, F_2 \sim \mathbb{P}_{E_3}
\end{align*}
Let ${\omega:\mathcal{I}_X \to \mathcal{I}_Y}$ be defined by
\begin{align*}
\omega: \begin{cases}
\nulli &\mapsto \doop(Y_1=0) \\
\doop(X_2=0) &\mapsto \nulli \\
\doop(X_1=0,\, X_2=0) &\mapsto \doop(Y_1=0) \\
\end{cases}
\end{align*}
Then it is true that ${\mathbb{P}_{\tau(X)}^{i} = \mathbb{P}_Y^{\doop(\omega(i))}}$ for all  ${i \in \mathcal{I}_X }$, while $\omega$ is not order-preserving and $\omega(\nulli)\not = \nulli$.
\end{example}

If the SEMs in the above example were used to model the same system, it would be problematic that the observational setting of $\mathcal{M}_X$---a description of the system when not having physically performed any intervention---would correspond to an interventional setting in $\mathcal{M}_Y$, conversely suggesting that the system \emph{had} been intervened upon.

To avoid the above conflict, we could demand in addition to surjectivity that $\omega$ map the null intervention of $\mathcal{M}_X$ to the null intervention of $\mathcal{M}_Y$.
This additional assumption would ensure commutativity of the left-hand part of the diagram in Theorem~\ref{lemma:commuting}.
However, as the following example shows, this would not ensure that the right-hand part of the diagram commutes for all pairs of interventions ${i \leq_X j}$, since in this case the arrow from $\mathbb{P}_Y^{\doop(\omega(i))}$ to $\mathbb{P}_Y^{\doop(\omega(j))}$ may not exist.\footnote{By definition of the poset $\mathcal{P}_Y$, this arrow exists if and only if $\omega(i) \leq_Y \omega(j)$.}

\begin{example}\label{example:wrong2}
Let $\mathcal{X},\mathcal{Y}$ and $\tau$ be as in Example~\ref{example:wrong1}. Consider the SEM $\mathcal{M}_X=\{\mathcal{S}_X , \mathcal{I}_X, \mathbb{P}_E\}$ where
\begin{align*}
\mathcal{S}_X = \big\{ & X_1 = E_1,\ X_2 = E_2, \ X_3 = X_1 + X_2 + E_3 \big\} \\
\mathcal{I}_X = \big\{ & \nulli,\ \doop(X_2=0),\ \doop(X_1=0,\, X_2=0)\big\}, \\
E_1 & = 1,\ \,  E_2 \sim \mathbb{P}_{E_2}, \  \, E_3 \sim \mathbb{P}_{E_3}
\end{align*}
where $\mathbb{P}_{E_2}$ and $\mathbb{P}_{E_3}$ are arbitrary distributions.
Let $\mathcal{M}_Y =\{\mathcal{S}_Y , \mathcal{I}_Y, \mathbb{P}_F\}$ be the SEM over $\mathcal{Y}$ with
\begin{align*}
\mathcal{S}_Y = \big\{ & Y_1 =1+ F_1,\ Y_2 = Y_1 + F_2 \big\} \\
\mathcal{I}_Y = \big\{ & \nulli,\ \doop(Y_1=0),\ \doop(Y_1=1) \big\}, \\
F_1 &\sim \mathbb{P}_{E_2},  \  \, F_2 \sim \mathbb{P}_{E_3}
\end{align*}
Let ${\omega:\mathcal{I}_X \to \mathcal{I}_Y}$ be defined by
\begin{align*}
\omega: \begin{cases}
\nulli &\mapsto \nulli  \\
\doop(X_2=0) &\mapsto \doop(Y_1=1) \\
\doop(X_1=0,\, X_2=0) &\mapsto \doop(Y_1=0) \\
\end{cases}
\end{align*}
Then it is true that ${\mathbb{P}_{\tau(X)}^{i} = \mathbb{P}_Y^{\doop(\omega(i))}}$ for all ${i \in \mathcal{I}_X}$ and $\omega(\nulli)=\nulli$, although $\omega$ is not order-preserving.
\end{example}

If the above SEMs were used as models of the same system, they would not suffer from the problem illustrated in Example~\ref{example:wrong1}.
Suppose now, however, that we have performed the intervention $\doop(X_2=0)$ in $\mathcal{M}_X$, corresponding to the intervention $\doop(Y_1=1)$ in $\mathcal{M}_Y$.
If we wish to reason about the effect of the intervention $\doop(X_1=0,\, X_2=0)$ in $\mathcal{M}_X$, we run into a problem.
$\mathcal{M}_X$ suggests that $\doop(X_1=0,\, X_2=0)$ could be implemented by performing an additional action on top of $\doop(X_2=0)$.
In contrast, $\mathcal{M}_Y$ suggests that implementing the corresponding intervention $\doop(Y_1=0)$ would conflict with the already performed intervention $\doop(Y_1=1)$.

\section{Examples of exact transformations}\label{sec:example-transformations}

In the introduction we motivated the problem considered in this paper by listing three settings in which differing model levels naturally occur.
Having now introduced the notion of an exact transformation between SEMs, we provide in this section examples of exact transformations falling into each of these categories.
The fact that a single framework can be used to draw an explicit correspondence between differing model levels in each of these settings demonstrates the generality of our framework.

Observe that in each of the following examples, the particular set of interventions considered is important. If we were to allow larger sets of interventions $\mathcal{I}_X$ in the SEM $\mathcal{M}_X$, the transformations given would not be exact. This highlights the importance to the causal modelling process of carefully considering the set of interventions. All proofs are found in  the Appendix.

\subsection{Marginalisation of variables}\label{sec:basic_trafos}

In the following two Theorems we consider two operations that can be performed on SEMs, namely marginalisation of childless or non-intervened variables, and prove that these are exact transformations.
That is, an SEM can be simplified into an SEM with fewer variables by either of these operations without losing any causal content concerning the remaining variables.

Thus if the SEM $\mathcal{M}_Y$ can be obtained from another SEM $\mathcal{M}_X$ by successively performing the operations in the following theorems, then $\mathcal{M}_Y$ is an exact transformation of $\mathcal{M}_X$ and hence the two models are causally consistent.
This formally explains why we can sensibly consider causal models that focus on a subsystem $\mathcal{M}_Y$ of a more complex system $\mathcal{M}_X$ (cf.\ Figure~\ref{fig:SEM_marginalisation}).
For a measure-theoretic treatment of marginalisation in SEMs, see~\cite{bongers2016structural}.

\begin{theorem}[Marginalisation of childless variables]\label{theorem:childless}
Let $\mathcal{M}_X=(\mathcal{S}_X,\mathcal{I}_X,\mathbb{P}_E)$ be an SEM and suppose that ${\mathbb{I}_Z\subset\mathbb{I}_X}$ is a set of indices of variables with no children, i.\,e.\ if $i\in\mathbb{I}_Z$ then $X_i$ does not appear in the right-hand side of any structural equation in $\mathcal{S}_X$.
Let $\mathcal{Y}$ be the set in which $Y = \left( X_i: i\in\mathbb{I}_X\setminus \mathbb{I}_Z \right)$ takes value.
Then the transformation $\tau: \mathcal{X} \to \mathcal{Y}$ mapping
\begin{align*}
   \tau: \left( x_i: i\in\mathbb{I}_X \right) = x &\mapsto y = \left( x_i: i\in\mathbb{I}_X\setminus \mathbb{I}_Z \right)
\end{align*}
naturally gives rise to an SEM $\mathcal{M}_Y$ that is an exact $\tau$-transformation of $\mathcal{M}_X$, corresponding to marginalising out the childless variables $X_i$ for $i\in\mathbb{I}_Z$.
\end{theorem}

\begin{theorem}[Marginalisation of non-intervened variables]\label{theorem:never_intervened}
Let $\mathcal{M}_X=(\mathcal{S}_X,\mathcal{I}_X,\mathbb{P}_E)$ be an \emph{acyclic} SEM and suppose that ${\mathbb{I}_Z\subset\mathbb{I}_X}$ is a set of indices of variables that are not intervened upon by any intervention $i\in\mathcal{I}_X$.
Let $\mathcal{Y}$ be the set in which $Y = \left( X_i: i\in\mathbb{I}_X\setminus \mathbb{I}_Z \right)$ takes value.
Then the transformation $\tau: \mathcal{X} \to \mathcal{Y}$ mapping
\begin{align*}
   \tau: \left( x_i: i\in\mathbb{I}_X \right) = x &\mapsto y = \left( x_i: i\in\mathbb{I}_X\setminus \mathbb{I}_Z \right)
\end{align*}
naturally gives rise to an SEM $\mathcal{M}_Y$ that is an exact $\tau$-transformation of $\mathcal{M}_X$, corresponding to marginalising out the never-intervened-upon variables $X_i$ for $i\in\mathbb{I}_Z$.
\end{theorem}

The assumption of acyclicity made in Theorem~\ref{theorem:never_intervened} can be relaxed to allow marginalisation of non-intervened variables in cyclic SEMs, at the expense of extra technical conditions (see Section~3 of \cite{bongers2016structural}).

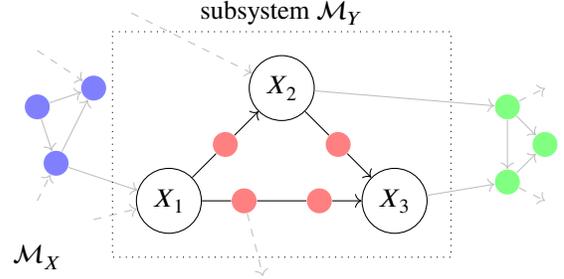
\begin{figure}
\center\
\begin{tikzpicture}

\node[keep] (X1) at(0,0) {$X_1$};
\node[keep] (X2) at(1.5,1.5) {$X_2$};
\node[keep] (X3) at(3,0) {$X_3$};

\node[drop1] (d1) at(4.5,.25) {};
\node[drop1] (d2) at(4.5,1.25) {};
\node[drop1] (d3) at(5,.75) {};

\draw[->,docolour] (X3) -- (d1);
\draw[->,docolour] (X2) -- (d2);
\draw[->,docolour] (d2) -- (d1);
\draw[->,docolour] (d1) -- (d3);
\draw[->,docolour] (d2) -- (d3);
\draw[->,docolour,dashed] (d1) -- (5,0);
\draw[->,docolour,dashed] (d2) -- (5,1.5);

\node[drop2a] (i12) at(.75,.75) {};
\node[drop2a] (i23) at(2.25,.75) {};
\node[drop2a] (i13a) at(1,0) {};
\node[drop2a] (i13b) at(2,0) {};

\draw[->] (X1) -- (i12) -- (X2);
\draw[->] (X2) -- (i23) -- (X3);
\draw[->] (X1) -- (i13a) -- (i13b) -- (X3);
\draw[->,docolour,dashed] (i13a) -- (1.25,-1);

\node[drop2b] (n1) at(-1.5,.5) {};
\node[drop2b] (n2) at(-1,1.5) {};
\node[drop2b] (n3) at(-1.75,1.25) {};

\draw[->,docolour] (n1) -- (X1);
\draw[->,docolour] (n3) -- (n1);
\draw[->,docolour] (n1) -- (n2);
\draw[->,docolour] (n3) -- (n2);
\draw[->,docolour,dashed] (-1.75,0) -- (n1);
\draw[->,docolour,dashed] (-1.75,2) -- (n2);
\draw[->,docolour,dashed] (-1,-.25) -- (X1);
\draw[->,docolour,dashed] (-.5,2.5) -- (X2);

\draw[dotted] (-.75,2.25) -- (3.75,2.25) -- (3.75,-.75) -- (-.75,-.75) -- (-.75,2.25);
\node (label) at(1.5,2.5) {subsystem $\mathcal{M}_Y$};
\node (label2) at(-1.75,-.75) {$\mathcal{M}_X$};

\end{tikzpicture}
\caption{Suppose that there is a complex model $\mathcal{M}_X$ but that we only wish to model the distribution over $X_1,X_2,X_3$ and how it changes under some interventions on $X_1,X_2,X_3$.
By Theorem~\ref{theorem:childless}, we can ignore downstream effects (\dropNode{drop1}) after grouping them together as one multivariate variable and by Theorem~\ref{theorem:never_intervened} we can ignore intermediate steps of complex mechanisms (\dropNode{drop2a}) and treat upstream causes as noise fluctuations (\dropNode{drop2b}).
That is, we can exactly transform the complex SEM $\mathcal{M}_X$ into a simpler model $\mathcal{M}_Y$ by marginalisation.
}
\label{fig:SEM_marginalisation}
\end{figure}

We remind the reader that our definition of an SEM does not require that the exogenous $E$-variables be independent.
Theorem~\ref{theorem:never_intervened} would not hold if this restriction were made (which is usually the case in the literature); marginalising out a common parent node will in general result in its children having dependent exogenous variables.

\subsection{Micro- to macro-level}\label{subsection:micromacro}

Transformations from micro- to macro-levels may arise in situations in which the micro-level variables can be observed via a `coarse' measurement device, represented by the function $\tau$, e.\,g.\ we can use a thermometer to measure the temperature of a gas, but not the motions of the individual particles. They may also arise due to deliberate modelling choice when we wish to describe a system using higher level features, e.\,g.\ viewing the motor cortex as a single entity responsible for movements, rather than as a collection of individual neurons.

In such situations, our framework of exact transformations allows one to investigate whether such a macro-level model admits a causal interpretation. The following theorem provides an exact transformation between a micro-level model $\mathcal{M}_X$ and a macro-level model $\mathcal{M}_Y$ in which the variables are aggregate features of variables in $\mathcal{M}_X$ obtained by averaging (cf.\ Figure~\ref{fig:micro_macro}).

\begin{theorem}[Micro- to macro-level]\label{theorem:micro-macro}
Let ${\mathcal{M}_X = \left(\mathcal{S}_X, \mathcal{I}_X, \mathbb{P}_{E,F} \right)}$ be a linear SEM over the variables ${W=\left( W_i \: : \: 1\leq  i \leq n \right)}$ and ${Z=\left( Z_i \: : \: 1\leq  i \leq m \right)}$ with
\begin{align*}
\mathcal{S}_X &= \left\lbrace W_i = E_i \:  : \: 1 \leq i \leq n  \right\rbrace \\
& \quad \ \  \cup \left\lbrace Z_i = \sum_{j=1}^n A_{ij}W_j  + F_{i} \:  : \: 1\leq i \leq m \right\rbrace \\
\mathcal{I}_X &= \Big{\{} \nulli, \ \doop(W= w), \ \doop(Z= z), \\
& \qquad\ \doop(W= w, Z= z ) :   w \in \mathbb{R}^{n}, \, z \in \mathbb{R}^m \Big{\}}
\end{align*}
and $(E,F)  \sim \mathbb{P}$ where $\mathbb{P}$ is any distribution over $\mathbb{R}^{n+m}$ and $A$ is a matrix.

Assume that there exists an $a\in \mathbb{R}$ such that each column of $A$ sums to $a$. Consider the following transformation that averages the $W$ and $Z$ variables:
\begin{align*}
\tau : \mathcal{X} &\rightarrow \mathcal{Y} = \mathbb{R}^2 \\
\begin{pmatrix} W \\ Z \end{pmatrix} &\mapsto \begin{pmatrix} \widehat{W} \\ \widehat{Z} \end{pmatrix} = \begin{pmatrix} \frac{1}{n}\sum_{i=1}^n W_i \\ \frac{1}{m}\sum_{j=1}^m Z_j  \end{pmatrix}
\end{align*}
Further, let $\mathcal{M}_Y = \left(\mathcal{S}_Y, \mathcal{I}_Y, \mathbb{P}_{\widehat{E},\widehat{F}} \right)$ over the variables ${\left\lbrace \widehat{W}, \widehat{Z} \right\rbrace}$ be an SEM with
\begin{align*}
\mathcal{S}_Y &= \Big\lbrace \widehat{W} = \widehat{E}, \ \widehat{Z} = \frac{a}{m}\widehat{W} + \widehat{F} \Big\rbrace \\
\mathcal{I}_Y &= \Big{\{} \nulli,\ \doop(\widehat{W}= \widehat{w}), \ \doop(\widehat{Z}= \widehat{z}), \\
& \qquad\ \doop(\widehat{W}= \widehat{w}, \widehat{Z}= \widehat{z} ) :   \widehat{w} \in \mathbb{R}, \, \widehat{z} \in \mathbb{R} \Big{\}} \\
\widehat{E}  & \sim \frac{1}{n}\sum_{i=1}^{n} E_i, \quad
\widehat{F}  \sim \frac{1}{m}\sum_{i=1}^{m} F_i
\end{align*}

Then $\mathcal{M}_Y$ is an exact $\tau$-transformation of $\mathcal{M}_X$.
\end{theorem}

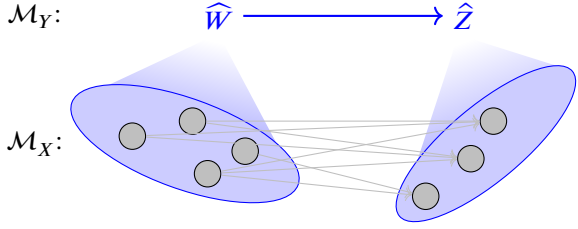
\begin{figure}
\center\
\begin{tikzpicture}
\node[ellipse, draw,blue, fill=blue!20, minimum width=3.2cm, minimum height=1.2cm,rotate=-20] (e1) at (0,0) {};
\node[ellipse, draw,blue, fill=blue!20, minimum width=3cm, minimum height=1cm,rotate=40] (e2) at (4,0) {};

\node[micro] (X1) at(-.7,.1) {};
\node[micro] (X2) at(0.1,0.3) {};
\node[micro] (X3) at(.3,-0.4) {};
\node[micro] (X4) at(.8,-0.1) {};

\node[micro] (Z1) at(4.1,.3) {};
\node[micro] (Z2) at(3.8,-.2) {};
\node[micro] (Z3) at(3.2,-.7) {};

\draw[->,docolour] (X1) -- (Z1);
\draw[->,docolour] (X1) -- (Z2);
\draw[->,docolour] (X2) -- (Z1);
\draw[->,docolour] (X2) -- (Z2);
\draw[->,docolour] (X3) -- (Z1);
\draw[->,docolour] (X3) -- (Z2);
\draw[->,docolour] (X3) -- (Z3);
\draw[->,docolour] (X4) -- (Z3);

\node[blue] (label1) at(.45,1.7) {$\widehat{W}$};
\node[blue] (label2) at(3.7,1.7) {$\widehat{Z}$};

\node (MY) at(-2,1.7) {$\mathcal{M}_Y$:};
\node (MX) at(-2,0) {$\mathcal{M}_X$:};

\begin{pgfonlayer}{background}
\shade[bottom color=white,top color=blue!50!white,shading angle=160] (e1.0)--(label1.275)--(label1.245)--(e1.180)--cycle;
\shade[bottom color=white,top color=blue!50!white,shading angle=200] (e2.0)--(label2.285)--(label2.255)--(e2.180)--cycle;
\end{pgfonlayer}

\draw[->,blue,thick] (label1) -- (label2) ;

\end{tikzpicture}
\caption{ An illustration of the setting considered in Theorem~\ref{theorem:micro-macro}. The micro-variables $W_1,\ldots,W_n$ and $Z_1,\ldots,Z_m$ in the SEM $\mathcal{M}_X$ can be averaged to derive macro-variables $\widehat{W}$ and $\widehat{Z}$ in such a way that the resulting macro-level SEM $\mathcal{M}_Y$ is an exact transformation of the micro-level SEM $\mathcal{M}_X$.}
\label{fig:micro_macro}
\end{figure}

\subsection{Stationary behaviour of dynamical processes}\label{subsec:stationary}

In this section we provide an example of an exact transformation between an SEM $\mathcal{M}_X$ describing a time-evolving system and another SEM $\mathcal{M}_Y$ describing the system after it has equilibrated. In this setting, $\tau$ could be thought of as representing our ability to only measure the time-evolving system at a single point in time, after the transient dynamics have taken place.

In particular, we consider a discrete-time linear dynamical system with identical noise and provide the explicit form of an SEM that models the distribution of the equilibria under each intervention (cf.\ Figure~\ref{fig:stationary}).%
\footnote{Note that the assumption that the transition dynamics be linear can be relaxed to more general non-linear mappings. In this case, however, the structural equations of $\mathcal{M}_Y$ can only be written in terms of implicit solutions to the structural equations of $\mathcal{M}_X$. For purposes of exposition, we stick here to the simpler case of linear dynamics.}

\begin{theorem}[Discrete-time linear dynamical process with identical noise]\label{theorem:identical}
Let $\mathcal{M}_X = \left(\mathcal{S}_X, \mathcal{I}_X, \mathbb{P}_{E} \right)$ over the variables ${\left\lbrace X_t^i \: : \: t \in \mathbb{Z}, \: i\in \{1,\ldots,n\} \right\rbrace}$ be a linear SEM with
\begin{align*}
\mathcal{S}_X &= \resizebox{.9\linewidth}{!}{$\displaystyle\left\lbrace X_{t+1}^i = \sum_{j=1}^n A_{ij}X_t^j + E_t^i \:  : \: i \in \{1,\ldots,n\}, t\in\mathbb{Z} \right\rbrace$} \\
&\qquad\text{i.\,e.}\ X_{t+1} = AX_t + E_t \\
\mathcal{I}_X &= \resizebox{0.43\textwidth}{!}{$\Big{\{} \doop(X_t^j= x_j \enspace \forall t \in \mathbb{Z},\forall j \in J):  x \in \mathbb{R}^{|J|}, \: J \subseteq \{1,\ldots,n\} \Big{\}} $}\\
E_t &= E\ \forall t\in\mathbb{Z} \text{ where } E \sim \mathbb{P}
\end{align*}
where $\mathbb{P}$ is any distribution over $\mathbb{R}^n$ and $A$ is a matrix.

Assume that the linear mapping $v\mapsto Av$ is a contraction.
Then the following transformation is well-defined under any intervention $i\in\mathcal{I}_X$:\footnote{In Appendix~\ref{contractionmapping:appendix} we show that $A$ being a contraction mapping ensures that the sequence $(X_t)_{t\in\mathbb{Z}}$ defined by $\mathcal{M}_X$ converges everywhere under any intervention $i\in\mathcal{I}_X$. That is, for any realisation $(x_t)_{t\in\mathbb{Z}}$ of this sequence, its limit $\lim_{t\rightarrow \infty}x_t$ as a sequence of elements of $\mathbb{R}^n$ exists.}
\begin{align*}
\tau : \mathcal{X} &\rightarrow \mathcal{Y} \\
(x_t)_{t\in \mathbb{Z}} & \mapsto y= \lim_{t\rightarrow \infty} x_t
\end{align*}
Let ${\mathcal{M}_Y = \left(\mathcal{S}_Y, \mathcal{I}_Y, \mathbb{P}_{F} \right)}$ be the (potentially cyclic) SEM over the variables ${\left\lbrace Y^i \: :  \: i\in \{1,\ldots,n\} \right\rbrace}$  with
\begin{align*}
\mathcal{S}_Y &= \left\lbrace Y^i = \frac{\sum_{j\not=i} A_{ij}Y^j}{1-A_{ii}} + \frac{F^i}{1-A_{ii}} \:  : \: i \in \{1,\ldots,n\} \right\rbrace \\
\mathcal{I}_Y &= \resizebox{.9\linewidth}{!}{$\displaystyle\Big{\{} \doop(Y^j= y_j \ \forall j \in J) : y \in \mathbb{R}^{|J|}, \: J \subseteq \{1,\ldots,n\} \Big{\}}$} \\
F &\sim \mathbb{P}
\end{align*}
Then $\mathcal{M}_Y$ is an exact $\tau$-transformation of $\mathcal{M}_X$.
\end{theorem}

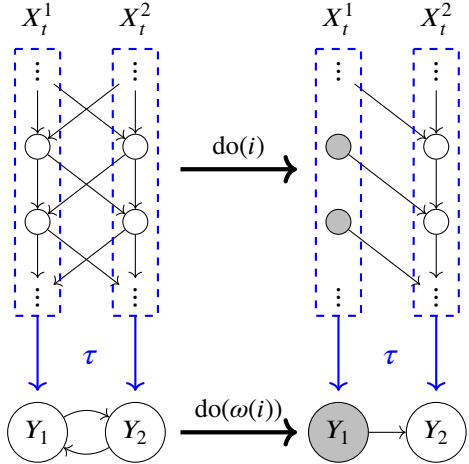
\begin{figure}
\center\
\begin{tikzpicture}
\def\hsep{1.3}
\def\copysep{4}
\foreach \i in {0,...,3} {
    \ifthenelse{\i=0 \OR \i=3}{
        \node (X\i) at(0,3-\i) {$\vdots$};
        \node (Y\i) at(\hsep,3-\i) {$\vdots$};
    }{
        \node[draw,circle] (X\i) at(0,3-\i) {};
        \node[draw,circle] (Y\i) at(\hsep,3-\i) {};
    }
}

\foreach \i in {0,...,2} {
    \pgfmathtruncatemacro{\ii}{\i + 1}
    \draw[->] (X\i) -- (X\ii);
    \draw[->] (X\i) -- (Y\ii);
    \draw[->] (Y\i) -- (X\ii);
    \draw[->] (Y\i) -- (Y\ii);
}
\foreach \i in {0,...,3} {
    \ifthenelse{\i=0 \OR \i=3}{
        \node (XX\i) at(\copysep,3-\i) {$\vdots$};
        \node (YY\i) at(\copysep+\hsep,3-\i) {$\vdots$};
    }{
        \node[draw,circle,fill=docolour]  (XX\i) at(\copysep,3-\i) {};
        \node[draw,circle] (YY\i) at(\copysep+\hsep,3-\i) {};
    }
}
\foreach \i in {0,...,2} {
    \pgfmathtruncatemacro{\ii}{\i + 1}
    \draw[->] (XX\i) -- (YY\ii);
    \draw[->] (YY\i) -- (YY\ii);
}
\node[draw,circle] (A) at(0,-1.8) {$Y_1$};
\node[draw,circle] (B) at(\hsep,-1.8) {$Y_2$};
\draw [->] (A) to [out=30,in=150] (B);
\draw [->] (B) to [out=210,in=-30] (A);
\node[draw,circle,fill=docolour] (AA) at(\copysep,-1.8) {$Y_1$};
\node[draw,circle] (BB) at(\copysep+\hsep,-1.8) {$Y_2$};
\draw[->] (AA) -- (BB);
\draw[->,ultra thick] (1.9,1.7) -- node[above] {$\doop(i)$} (3.45,1.7);
\draw[->,ultra thick] (1.9,-1.8) -- node[above] {$\doop(\omega(i))$} (3.45,-1.8);
\draw[dashed,thick,blue] (-.3,3.3) -- (0.3,3.3) -- (0.3,-.25) -- (-.3,-.25) -- (-.3,3.3);
\draw[dashed,thick,blue] (\hsep-.3,3.3) -- (\hsep+0.3,3.3) -- (\hsep+0.3,-.25) -- (\hsep-.3,-.25) -- (\hsep-.3,3.3);

\draw[dashed,thick,blue] (\copysep-.3,3.3) -- (\copysep+0.3,3.3) -- (\copysep+0.3,-.25) -- (\copysep-.3,-.25) -- (\copysep-.3,3.3);
\draw[dashed,thick,blue] (\copysep+\hsep-.3,3.3) -- (\copysep+\hsep+0.3,3.3) -- (\copysep+ \hsep+0.3,-.25) -- (\copysep+\hsep-.3,-.25) -- (\copysep+\hsep-.3,3.3);

\draw[->,blue,thick] (0,-.25) -- (0,-1.25);
\draw[->,blue,thick] (\hsep,-.25) -- (\hsep,-1.25);
\draw[->,blue,thick] (\copysep,-.25) -- (\copysep,-1.25);
\draw[->,blue,thick] (\copysep+\hsep,-.25) -- (\copysep+\hsep,-1.25);

\node[blue] (tau1) at(0.7,-.8) {\large$\tau$};
\node[blue] (tau2) at(4.7,-.8) {\large$\tau$};

\node (Xlab1) at(0,3.7) {$X^1_t$};
\node (Ylab1) at(\hsep,3.7) {$X^2_t$};

\node (Xlab2) at(\copysep,3.7) {$X^1_t$};
\node (Ylab2) at(\copysep+\hsep,3.7) {$X^2_t$};

\end{tikzpicture}
\caption{An illustration of the setting considered in Theorem~\ref{theorem:identical}. The discrete-time dynamical process is exactly transformed into a model describing its equilibria.}
\label{fig:stationary}
\end{figure}

The above theorem demonstrates how a linear additive SEM can arise as a result of making observations of a dynamical process.
This supports one interpretation of SEMs as a description of a dynamical process that equilibrates quickly compared to its external environment.\footnote{This interpretation corresponds to the assumption that the noise in the dynamical model is constant through time, and is used by e.\,g.\ \cite{lacerda2012discovering,Mooij_et_al_NIPS_11,hyttinen2012learning,mooij2013ode} and \cite{mooij2013cyclic} to meaningfully interpret cyclic SEMs.}
The framework of exact transformations allows us to explain in a precise way the sense in which such equilibrium models can be used as causal descriptions of an underlying dynamical process.

This result also sheds light on the interpretation of cyclic causal models.
One interpretation of the structural equations of an acyclic SEM is that they represent a temporally ordered series of mechanisms by which data are generated.
This is not possible in the case that the SEM exhibits cycles: there does not exist a partial ordering on the variables and hence one cannot think of each variable being generated temporally downstream of its parents.
By showing that cyclic SEMs can arise as exact transformations of \emph{acyclic} SEMs, we provide an interpretation of cyclic SEMs that does not suffer from the above problem.

\section{Discussion and Future work}\label{sec:questions}

It's turtles all the way down!
There is no such thing as a `correct' model, but in this paper we introduced the notions of exact transformations between SEMs to evaluate when two SEMs can be viewed as causally consistent models of the same system.
Illustrating how these notions can be used in order to relate differing model levels, we proved in Section~\ref{sec:example-transformations} the exactness of transformations occurring in three different settings.
These have implications for the following questions in causal modelling: When can we model only a subsystem of a more complex system? When does a micro-level system admit a causal description in terms of macro-level features? How do cyclic causal models arise?

Our work has implications for other problems in causal modelling.
It suggests that ambiguous manipulations~\citep{spirtes2004causal} may be thought of as arising due to the application of an inexact transformation to an SEM $\mathcal{M}_X$.
This was illustrated in Section~\ref{sec:cholesterol} in which LDL and HDL cholesterol were only measured via their sum TC, resulting in a model that suffered from the problem of ambiguous manipulations (cf.\ Figure~\ref{fig:cholesterol:a}) since it was not an exact transformation of the underlying model (cf.\ Figure~\ref{fig:cholesterol:b}).
This is related to the problem of causal variable definition as studied by~\cite{eberhardt2016green}.

A future line of enquiry would be to generalise the notion of an exact transformation in order to analyse the trade-off between model accuracy and model complexity for causal modelling using SEMs.
For a transformation to be exact, we require that the posets $\mathcal{P}_{\tau(X)}$ and $\mathcal{P}_Y$ be equal. One could imagine a `softening' of this requirement such that the distributions in the posets are required to be only approximately equal.
A slightly inaccurate model with a small number of variables may be preferable to an accurate but complex model.

We discussed the importance of an order-preserving $\omega$ to ensure a notion of causal consistency between two SEMs.
It would be interesting to better understand the conditions under which different properties of consistency between causal models hold -- for instance, counterfactual reasoning, which we have not discussed in this paper.

While we have introduced the notion of an exact transformation, we have not provided any criterion to choose from amongst the set of all possible exact transformations of an SEM. Foundational work in a similar direction to ours has been done by \cite{chalupka2015visual,chalupka2016multi}, who consider a particular discrete setting.
They provide algorithms to learn a transformation of a micro-level model to a macro-level model with desirable information-theoretic properties.
We conjecture that our framework may lead to extensions of their work, e.\,g.\ to the continuous setting.

Finally, suppose that we have made observations of an underlying system $\mathcal{M}_X$ via a measurement device $\tau$, and that we want to fit an SEM $\mathcal{M}_Y$ from a restricted model class to our data.
By using our framework, asking whether or not $\mathcal{M}_Y$ admits a causal interpretation consistent with $\mathcal{M}_X$ reduces to asking whether the transformation is exact.
More generally, by fixing any two of $\mathcal{M}_X$, $\tau$ and $\mathcal{M}_Y$, we can ask what properties must be fulfilled by the third in order for the two models to be causally consistent.
We hope that this may lead to the practical use of SEMs being theoretically grounded.

\subsubsection*{Acknowledgements}

We thank Tobias Mistele for valuable early feedback.
Stephan Bongers was supported by NWO, the Netherlands Organization for Scientific Research (VIDI grant 639.072.410).
This project has received funding from the European Research Council (ERC) under the European Union's Horizon 2020 research and innovation programme (grant agreement n$^{\mathrm{o}}$ 639466).

\begingroup
\renewcommand{\section}[2]{\subsubsection#1{#2}}
\bibliography{references}
\endgroup

\clearpage
\onecolumn
\appendix
{\bf\Large Appendix}

\section{Proofs for Section~\ref{sec:exact_transformation_sem}: elementary exact transformations}\label{first_properties:appendix}

{
\renewcommand{\thedefinition}{\ref{lemma:elementary}}
\begin{lemma}
The identity mapping and permuting the labels of variables are both exact transformations.
That is, if $\mathcal{M}_X$ is an SEM and $\pi:\mathbb{I}_X \to \mathbb{I}_X$ is a bijection then the transformation
\begin{align*}
\tau:\mathcal{X}&\to\mathcal{Y}\\
(x_i:i\in\mathbb{I}_X) &\mapsto (x_{\pi(i)}:i\in\mathbb{I}_X)
\end{align*}
naturally gives rise to an SEM $\mathcal{M}_Y$ that is an exact $\tau$-transformation of $\mathcal{M}_X$, corresponding to relabelling the variables.
\end{lemma}
\addtocounter{definition}{-1}
}
\begin{proof}[Proof of Lemma~\ref{lemma:elementary}]
Consider the SEM $\mathcal{M}_Y$ obtained from $\mathcal{M}_X$ by replacing, for all $i\in\mathbb{I}_X$, any occurrence of $X_i$ in the structural equations $\mathcal{S}_X$ and interventions $\mathcal{I}_X$ by $Y_{\pi(i)}$ and leaving the distribution over the exogenous variables unchanged.
\end{proof}

\begin{proof}[Proof of Lemma~\ref{theorem:transitivity} (Transitivity of exact transformations)]
    Let $\omega_{ZY}:\mathcal{I}_Y \to \mathcal{I}_Z$ and $\omega_{YX}:\mathcal{I}_X \to \mathcal{I}_Y$ be the mappings between interventions corresponding to the exact transformations $\tau_{ZY}$ and $\tau_{YX}$ respectively and define $\omega_{ZX} = \omega_{ZY}\circ\omega_{YX}:\mathcal{I}_X \to \mathcal{I}_Z$.
    Then $\omega_{ZX}$ is surjective and order-preserving since both $\omega_{ZY}$ and $\omega_{YX}$ are surjective and order-preserving.
    Since $\tau_{ZY}$ and $\tau_{YX}$ are exact it follows that for all $i\in\mathcal{I}_X$
    \begin{align*}
        \mathbb{P}^{i}_{\tau_{ZX}(X)}
        =
        \mathbb{P}^{ \omega_{ZY}(\omega_{YX}(i))}_{\tau_{ZY}(\tau_{YX}(X))}
        =
        \mathbb{P}^{\doop(\omega_{ZX}(i))}_{Z}
    \end{align*}
    i.\,e.\ $\mathcal{M}_Z$ is an $\tau_{ZX}$-exact transformation of $\mathcal{M}_X$.
\end{proof}

\section{Proofs for Section~\ref{sec:basic_trafos}: Marginalisation of variables}\label{marginalisation:appendix}

\begin{proof}[Proof of Theorem~\ref{theorem:childless} (Marginalisation of childless variables)]
By Lemma~\ref{theorem:transitivity} it suffices to proof this for marginalisation of one childless variable.
Without loss of generality, let $X_1$ be the childless variable to be marginalised out.

Let $\mathcal{M}_Y=(\mathcal{S}_Y,\mathcal{I}_Y,\mathbb{P}_F)$ be the SEM where
\begin{itemize}
    \item the structural equations $\mathcal{S}_Y$ are obtained from $\mathcal{S}_X$ by removing the structural equation corresponding to the childless variable $X_1$;
    \item $\mathcal{I}_Y$ is the image of the map $\omega:\mathcal{I}_X \to \mathcal{I}_Y$ that drops any reference to the variable $X_1$ (e.\,g.\ ${\doop( X_1=x_1, X_2=x_2) \in \mathcal{I}_X}$ would be mapped to $\doop(X_2=x_2) \in \mathcal{I}_Y$);
    \item $F = (E_i:\ i \in \mathbb{I}_X\setminus\{1\})$ are the remaining noise variables distributed according to their marginal distribution under $\mathbb{P}_E$.
\end{itemize}
By construction, $\omega$ is surjective and order-preserving.
Let $i\in\mathcal{I}_X$ be any intervention.
The variable $X_1$ being childless ensures that the law on the remaining variables $X_k,k\in\mathbb{I}_X\setminus\{1\}$ that we obtain by \emph{marginalisation} of the childless variable, i.\,e.\ $\mathbb{P}_{\tau(X)}^{i}$, is equivalent to the law one obtains by simply \emph{dropping} the childless variable, which is exactly what the law under $\mathcal{M}_Y$ amounts to, i.\,e.\ $\mathbb{P}_Y^{\omega({\doop(i)})}$.
\end{proof}

\begin{proof}[Proof of Theorem~\ref{theorem:never_intervened} (Marginalisation of non-intervened variables)]
By Lemma~\ref{theorem:transitivity} it suffices to proof this for marginalisation of one never-intervened-upon variable.
Without loss of generality, let $X_1$ be the never-intervened-upon variable to be marginalised out.
By acyclicity of the SEM $\mathcal{M}_X$, the structural equation corresponding to variable $X_1$ is of the form $X_1 = f_1\left(\mathbf{X}_{\pa(1)},E_1\right)$ and $X_1$ does not appear in the structural equation for any of its ancestors.

Now let $\mathcal{M}_Y=(\mathcal{S}_Y,\mathcal{I}_Y,\mathcal{P}_F)$ be the SEM where
\begin{itemize}
    \item $\mathcal{I}_Y = \mathcal{I}_X$;
    \item $F_i = ((E_i,E_1):\ i \in \mathbb{I}_X\setminus\{1\})$ are the noise variables distributed as implied by $\mathbb{P}_E$;
    \item the structural equations $\mathcal{S}_Y$ are obtained from $\mathcal{S}_X$ by removing the structural equation of $X_1$ and replacing any occurrence of $X_1$ in the right-hand side of the structural equations of children of $X_1$ by $f_1\left(\mathbf{X}_{\pa(1)},E_1\right)$, yielding $X_i = f_i\left(f_1\left(\mathbf{X}_{\pa(1)},E_1\right),\ \mathbf{X}_{\pa(i)},\ E_i\right)$.
\end{itemize}
Note that the structural equations  of the resulting SEM are still acyclic and are all of the form $X_i = h_i\left(\mathbf{X}_{\setminus i},\ F_i\right)$.

Then $\mathcal{M}_Y$ is, by construction, an $\tau$-exact transformation of $\mathcal{M}_X$ for $\omega=\operatorname{id}$.
\end{proof}

\section{Proof for Section~\ref{subsection:micromacro}: Micro- to macro-level}\label{micromacro:appendix}

\begin{proof}[Proof of Theorem~\ref{theorem:micro-macro}]
We begin by defining a mapping between interventions
\begin{align*}
\omega : \mathcal{I}_X &\rightarrow \mathcal{I}_Y \\
\nulli &\mapsto \nulli \\
\doop(W= w)&\mapsto \doop\left(\widehat{W} = \frac{1}{n}\sum_{i=1}^n w_i\right) \\
\doop(Z= z)&\mapsto \doop\left(\widehat{Z} = \frac{1}{m}\sum_{i=1}^m z_i\right) \\
\doop(W=w, Z=z)&\mapsto \doop\left(\widehat{W} = \frac{1}{n}\sum_{i=1}^n w_i,\, \widehat{Z} = \frac{1}{m}\sum_{i=1}^m z_i\right)
\end{align*}

Note that $\omega$ is surjective and order-preserving (in fact, it is an order embedding).
Therefore, it only remains to show that the distributions implied by $\tau(X)$ under any intervention $i\in\mathcal{I}_X$ agree with the corresponding distributions implied by $\mathcal{M}_Y$.
That is, we have to show that
\[ \mathbb{P}_{\tau(X)}^{i} = \mathbb{P}_{Y}^{\doop(\omega(i))} \quad \forall i\in\mathcal{I}_X \]
In the observational setting, the distribution over $\mathcal{Y}$ is implied by the following equations:
\begin{align*}
\widehat{W} &= \frac{1}{n}\sum_{i=1}^n W_i = \frac{1}{n}\sum_{i=1}^n E_i \\
\widehat{Z} &= \frac{1}{m}\sum_{i=1}^m Z_i =  \frac{1}{m}\sum_{i=1}^m \left( \sum_{j=1}^n A_{ij}W_j  + F_i\right) = \frac{a}{m}\widehat{W} +  \frac{1}{m}\sum_{i=1}^m F_i
\end{align*}
Since the distributions of the exogenous variables in $\mathcal{M}_Y$ are given by $\widehat{E} \sim \frac{1}{n}\sum_{i=1}^n E_i$, $\widehat{F} \sim \frac{1}{m}\sum_{i=1}^m F_i$, it follows that $\mathbb{P}_{\tau(X)}^{\doop(\nulli)}$  and $\mathbb{P}_{Y}^{\doop(\nulli)}$ agree. Similarly, the push-forward measure on $\mathcal{Y}$ induced by the intervention $\doop(W=w)\in\mathcal{I}_X$ is given by
\begin{align*}
\widehat{W} &= \frac{1}{n}\sum_{i=1}^n W_i = \frac{1}{n}\sum_{i=1}^n w_i \\
\widehat{Z} &= \frac{1}{m}\sum_{i=1}^m Z_i =  \frac{1}{m}\sum_{i=1}^m \left( \sum_{j=1}^n A_{ij}W_j  + F_i\right) = \frac{a}{m}\widehat{W} +  \frac{1}{m}\sum_{i=1}^m F_i
\end{align*}
which is the same as the distribution induced by the $\omega$-corresponding intervention $\doop\left(\widehat{W} = \frac{1}{n}\sum_{i=1}^n w_i\right)$ in $\mathcal{M}_Y$.

Similar reasoning shows that this also holds for the interventions $\doop(Z=z)$ and $\doop(W=w, Z=z)$.

\end{proof}

\section{Proof for Section~\ref{subsec:stationary}: stationary behaviour of dynamical processes}\label{theorem:identical:appendix}

\begin{proof}[Proof of Theorem~\ref{theorem:identical}]
We begin by defining a mapping between interventions
\begin{align*}
\omega : \mathcal{I}_X &\rightarrow \mathcal{I}_Y \\
\doop(X_t^j= x_j \enspace \forall t \in \mathbb{Z}, \: \forall j \in J) & \mapsto \doop(Y^j= x_j \ \forall j \in J)
\end{align*}
Note that $\omega$ is surjective and order-preserving (in fact, it is an order embedding).
Therefore, it only remains to show that the distributions implied by $\tau(X)$ under any intervention $i\in\mathcal{I}_X$ agree with the corresponding distributions implied by $\mathcal{M}_Y$.
That is, we have to show that
\[ \mathbb{P}_{\tau(X)}^{i} = \mathbb{P}_{Y}^{\doop(\omega(i))} \quad \forall i\in\mathcal{I}_X \]
For this we consider, without loss of generality, the distribution arising from performing the $\mathcal{M}_X$-level intervention
\[ i = \doop(X_t^j=x_j\ \forall t\in\mathbb{Z},\forall j \leq m \leq n) \in \mathcal{I}_X \]
for $m \in [n]$ (for $m=0$ this amounts to the null-intervention).

Since $A$ is a contraction mapping, it follows from Lemma~\ref{lemma:contraction_convergence} that for any intervention in $\mathcal{I}_X$, the sequence of random variables $X_t$ defined by $\mathcal{M}_X$ converges everywhere.
That is, there exists a random variable $X_*$ such that ${X_t \xrightarrow[t\to\infty]{\text{everywhere}} X_*}$.
In the case of the intervention $i$ above, the random variable $X_*$ satisfies:
\begin{align}\label{eq:x-star-cases}
\begin{cases}
X^k_{*} = x_k & \text{if}\  k \leq m \\
X^k_{*} = \sum_j A_{kj}X^j_{*} + E^k  & \text{if}\ m < k \leq n
\end{cases}
\end{align}
Since $\tau(X) = \lim_{t\rightarrow \infty}X_t$, it follows from the definition of $X_*$ that $\tau(X)= X_*$, and hence $\tau(X)$ also satisfies the equations above.
It follows (rewriting the second line in Equation \ref{eq:x-star-cases} above) that under the push-forward measure $\mathbb{P}_{\tau(X)}^{i} = \tau\left(\mathbb{P}_X^{\doop(i)}\right)$ the distribution of the random variable $\tau(X)=X_*$ is given by:
\begin{align*}
\begin{cases}
X_*^k = x_k & \text{if}\ k \leq m \\
X_*^k = \frac{\sum_{j\neq k} A_{kj}X_*^j}{1-A_{kk}} + \frac{E^k}{1-A_{kk}} & \text{if}\ m < k \leq n
\end{cases}
\end{align*}
We need to compare this to the law of $Y$ as implied by $\mathcal{M}_Y$ under the intervention $\omega(i)$, i.\,e.\ $\mathbb{P}_Y^{\doop(\omega(i))}$.
The $\mathcal{M}_Y$-level intervention $\omega(i)$ corresponding to $i$ is
\[ \omega(i) = \doop(Y^j=x_j\ \forall j\leq m\leq n) \in \mathcal{I}_Y \]
and so the structural equations of $\mathcal{M}_Y$ under the intervention $\omega(\doop(i))$ are
\begin{align*}
\begin{cases}
Y^k = x_k  & \text{if}\ k\leq m \\
Y^k = \frac{\sum_{j\neq k} A_{kj}Y^j}{1-A_{kk}} + \frac{F^k}{1-A_{kk}}  & \text{if}\ m <  k \leq n
\end{cases}
\end{align*}
Since $F\sim E$ it indeed follows that $\tau(X) \sim Y$, i.\,e.\ $\mathbb{P}_{\tau(X)}^{i} = \mathbb{P}_Y^{\doop(\omega(i))}$.

Thus $\mathcal{M}_Y$ is an exact $\tau$-transformation of $\mathcal{M}_X$.
\end{proof}

\subsection{Contraction mapping and convergence}\label{contractionmapping:appendix}

The following Lemmata show that $A$ being a contraction mapping ensures that the sequence $(X_t)_{t\in\mathbb{Z}}$ defined by $\mathcal{M}_X$ in Theorem~\ref{theorem:identical} converges everywhere under any intervention $i\in\mathcal{I}_X$.
That is, for any realisation $(x_t)_{t\in\mathbb{Z}}$ of this sequence, its limit $\lim_{t\rightarrow \infty}x_t$ as a sequence of elements of $\mathbb{R}^n$ exists.

\begin{lemma}\label{lemma:contraction_add}
Suppose that the function
\begin{align*}
f: \mathbb{R}^n & \rightarrow \mathbb{R}^m \\
 x & \mapsto f(x)
\end{align*}
is a contraction mapping.
Then, for any $e \in \mathbb{R}^m$, so is the function
\begin{align*}
f^*: \mathbb{R}^n & \rightarrow \mathbb{R}^m \\
 x & \mapsto f(x) + e
\end{align*}
\end{lemma}
\begin{proof}
By definition, there exists $c<1$ such that for any $x,y\in \mathbb{R}^n$,
\[
\| f^*(x)-f^*(y) \| = \| (f(x)+e)-(f(y)+e) \| = \| f(x)-f(y) \| \leq c \| x - y \|
\]
and hence $f^*$ is a contraction mapping.
\end{proof}

\begin{lemma}\label{lemma:contraction_intervene}
Suppose that the function
\begin{align*}
f: \mathbb{R}^n & \rightarrow \mathbb{R}^n \\
 x = \begin{pmatrix}x_1\\ \vdots \\ x_n \end{pmatrix}
 & \mapsto \begin{pmatrix}f_1(x)\\ \vdots \\ f_n(x) \end{pmatrix}
\end{align*}
is a contraction mapping. Then for any $m \leq n$, and $x^*_i \in \mathbb{R}, \: i \in [m]$, so is the function
\begin{align*}
f^*: \mathbb{R}^n & \rightarrow \mathbb{R}^n \\
 x = \begin{pmatrix}x_1\\ \vdots \\ x_n \end{pmatrix}
 & \mapsto \begin{pmatrix}x^*_1\\ \vdots \\ x^*_m \\ f_{m+1}(x) \\ \vdots \\ f_n(x) \end{pmatrix}
\end{align*}
\end{lemma}
\begin{proof}
By definition, there exists $c < 1$ such that for any $x,y\in \mathbb{R}^n$,
\begin{align*}
\| f^*(x) - f^*(y) \| =
\left\Vert
\begin{pmatrix}x^*_1\\ \vdots \\ x^*_m \\ f_{m+1}(x) \\ \vdots \\ f_n(x) \end{pmatrix} - \begin{pmatrix}x^*_1\\ \vdots \\ x^*_m \\ f_{m+1}(y) \\ \vdots \\ f_n(y) \end{pmatrix} \right\Vert =
\left\Vert
\begin{pmatrix} 0 \\ \vdots \\ 0 \\ f_{m+1}(x) - f_{m+1}(y) \\ \vdots \\ f_{n}(x) - f_n(y) \end{pmatrix} \right\Vert &\leq
\left\Vert
\begin{pmatrix} f_{1}(x) - f_1(y)  \\ \vdots \\ f_{n}(x) - f_n(y) \end{pmatrix} \right\Vert \\
& = \| f(x)-f(y) \|    \\
&\leq c \| x - y \|
\end{align*}
and hence $f^*$ is a contraction mapping.
\end{proof}

\begin{lemma}\label{lemma:contraction_convergence}
Consider the SEM $\mathcal{M}_X$ in Theorem~\ref{theorem:identical}, and suppose that the linear map $A:\mathbb{R}^n \to \mathbb{R}^n$ is a contraction mapping.
Then, for any intervention $i\in\mathcal{I}_X$, the sequence of $X_t$ converges everywhere.
\end{lemma}
\begin{proof}
Consider, without loss of generality, the intervention
\[ \doop(X_t^j=x_j\ \forall t\in\mathbb{Z},\forall j \leq m \leq n) \in \mathcal{I}_X \]
for $m \in [n]$ (for $m=0$ this amounts to the null-intervention).
The structural equations under this intervention are
\begin{align*}
\begin{cases}
X^k_{t+1} = x_k \quad &\text{if} \:  k \leq m \\
X^k_{t+1} = \sum_j A_{kj}X^j_{t} + E^k  \quad &\text{if} \: m < k \leq n
\end{cases}
\end{align*}
and thus the sequence $X_t$ can be seen to transition according to the function $f = g\circ h$, where
\begin{align*}
    h: \mathbb{R}^n &\to \mathbb{R}^n \\
    v &\mapsto w = Av + E \\
    \\
    g : \mathbb{R}^n &\to \mathbb{R}^n \\
    w = \begin{pmatrix}w_1\\\vdots\\w_n\end{pmatrix} &\mapsto
    \begin{pmatrix} x_1\\\vdots\\x_m\\w_{m+1}\\\vdots\\w_n \end{pmatrix}
\end{align*}
By Lemma~\ref{lemma:contraction_add} and Lemma~\ref{lemma:contraction_intervene}, $f$ is a contraction mapping for any fixed $E$.
Thus, by the contraction mapping theorem, the sequence of $X_t$ converges everywhere to a unique fixed point.
\end{proof}

\bibliographyappendix{references}

\end{document}